\documentclass[a4paper,12pt]{amsart}

\usepackage{graphicx,natbib,amssymb,datetime}

\newtheorem{thm}{Theorem}
\newtheorem{lem}{Lemma}

\def\ci{\!\perp\!}
\def\nci{\!\not\perp\!}
\newcommand{\comments}[1]{}

\begin{document}

\title[]{Towards Optimal Learning of Chain Graphs}

\author[]{Jose M. Pe\~{n}a\\
ADIT, Department of Computer and Information Science\\
Link\"oping University, SE-58183 Link\"{o}ping, Sweden\\
jose.m.pena@liu.se}

\date{\currenttime, \ddmmyydate{\today}}

\begin{abstract}
In this paper, we extend Meek's conjecture \citep{Meek1997} from directed and acyclic graphs to chain graphs, and prove that the extended conjecture is true. Specifically, we prove that if a chain graph $H$ is an independence map of the independence model induced by another chain graph $G$, then (i) $G$ can be transformed into $H$ by a sequence of directed and undirected edge additions and feasible splits and mergings, and (ii) after each operation in the sequence $H$ remains an independence map of the independence model induced by $G$. Our result has the same important consequence for learning chain graphs from data as the proof of Meek's conjecture in \citep{Chickering2002} had for learning Bayesian networks from data: It makes it possible to develop efficient and asymptotically correct learning algorithms under mild assumptions.
\end{abstract}

\maketitle

\section{Preliminaries}\label{sec:preliminaries}

In this section, we review some concepts from probabilistic graphical models that are used later in this paper. See, for instance, \citep{Lauritzen1996} and \citep{Studeny2005} for further information. All the graphs and independence models in this paper are defined over a finite set $V$. All the graphs in this paper are hybrid graphs, i.e. they have (possibly) both directed and undirected edges. We assume throughout the paper that the union and the intersection of sets precede the set difference when evaluating an expression.

If a graph $G$ has a directed (resp. undirected) edge between two nodes $X_{1}$ and $X_{2}$, then we write that $X_{1} \rightarrow X_{2}$ (resp. $X_{1} - X_{2}$) is in $G$. When there is a directed or undirected edge between two nodes of $G$, we say that the two nodes are adjacent in $G$. The parents of a set of nodes $Y$ of $G$ is the set $Pa_G(Y) = \{X_1 | X_1 \rightarrow X_2$ is in $G$ and $X_2 \in Y \}$. The neighbors of a set of nodes $Y$ of $G$ is the set $Ne_G(Y) = \{X_1 | X_1 - X_2$ is in $G$ and $X_2 \in Y \}$. The boundary of a node $X_2$ of $G$ is the set $Bd_G(X_2) = Pa_G(X_2) \cup Ne_G(X_2)$. A route between two nodes $X_{1}$ and $X_{n}$ of $G$ is a sequence of nodes $X_{1}, \ldots, X_{n}$ st $X_{i}$ and $X_{i+1}$ are adjacent in $G$ for all $1 \leq i < n$. The length of a route is the number of (not necessarily distinct) edges in the route. We treat all singletons as routes of length zero. A route in $G$ is called undirected if $X_{i} - X_{i+1}$ is in $G$ for all $1 \leq i < n$. A route in $G$ is called descending from $X_1$ to $X_n$ if $X_{i} - X_{i+1}$ or $X_{i} \rightarrow X_{i+1}$ is in $G$ for all $1 \leq i < n$. If there is a descending route from $X_{1}$ to $X_{n}$ in $G$, then $X_{n}$ is called a descendant of $X_{1}$. Note that $X_{1}$ is a descendant of itself, since we allow routes of length zero. The descendants of a set of nodes $Y$ of $G$ is the union of the descendants of each node of $Y$ in $G$. Given a route $\rho$ between $X_1$ and $X_n$ in $G$ and a route $\rho'$ between $X_n$ and $X_m$ in $G$, $\rho \cup \rho'$ denotes the route between $X_1$ and $X_m$ in $G$ resulting from appending $\rho'$ to $\rho$.

A chain is a partition of $V$ into ordered subsets, which we call the blocks of the chain. We say that an element $X \in V$ is to the left of another element $Y \in V$ in a chain $\alpha$ if the block of $\alpha$ containing $X$ precedes the block of $\alpha$ containing $Y$ in $\alpha$. Equivalently, we can say that $Y$ is to the right of $X$ in $\alpha$. We say that a graph $G$ and a chain $\alpha$ are consistent when (i) for every edge $X \rightarrow Y$ in $G$, $X$ is to the left of $Y$ in $\alpha$, and (ii) for every edge $X - Y$ in $G$, $X$ and $Y$ are in the same block of $\alpha$. A chain graph (CG) is a graph that is consistent with a chain. A set of nodes of a CG is connected if there exists an undirected route in the CG between every pair of nodes of the set. A component of a CG is a maximal (wrt set inclusion) connected set of its nodes. A block of a CG is a set of components of the CG st there is no directed edge between their nodes in the CG. Note that a component of a CG is connected, whereas a block of a CG or a block of a chain that is consistent with a CG is not necessarily connected. Given a set $K$ of components of $G$, a component $C \in K$ is called maximal in $G$ if none of its nodes is a descendant of $K \setminus \{C\}$ in $G$. A component $C$ of $G$ is called terminal in $G$ if its descendants in $G$ are exactly $C$. Let a component $C$ of $G$ be partitioned into two non-empty connected subsets $C \setminus L$ and $L$. By splitting $C$ into $C \setminus L$ and $L$ in $G$, we mean replacing every edge $X - Y$ in $G$ st $X \in C \setminus L$ and $Y \in L$ with an edge $X \rightarrow Y$. Moreover, we say that the split is feasible if (i) $X - Y$ is in $G$ for all $X, Y \in Ne_{G}(L) \cap (C \setminus L)$, and (ii) $X \rightarrow Y$ is in $G$ for all $X \in Pa_{G}(L)$ and $Y \in Ne_{G}(L) \cap (C \setminus L)$. Let $L$ and $R$ denote two components of $G$ st $Pa_G(R) \cap L \neq \emptyset$. By merging $L$ and $R$ in $G$, we mean replacing every edge $X \rightarrow Y$ in $G$ st $X \in L$ and $Y \in R$ with an edge $X - Y$. Moreover, we say that the merging is feasible if (i) $X - Y$ is in $G$ for all $X, Y \in Pa_{G}(R) \cap L$, and (ii) $X \rightarrow Y$ is in $G$ for all $X \in Pa_{G}(R) \setminus L$ and $Y \in Pa_{G}(R) \cap L$. 

A section of a route $\rho$ in a CG is a maximal undirected subroute of $\rho$. A section $X_{2} - \ldots - X_{n-1}$ of $\rho$ is a collider section of $\rho$ if $X_{1} \rightarrow X_{2} - \ldots - X_{n-1} \leftarrow X_{n}$ is a subroute of $\rho$. Moreover, the edges $X_{1} \rightarrow X_{2}$ and $X_{n-1} \leftarrow X_{n}$ are called collider edges. Let $X$, $Y$ and $Z$ denote three disjoint subsets of $V$. A route $\rho$ in a CG is said to be $Z$-active when (i) every collider section of $\rho$ has a node in $Z$, and (ii) every non-collider section of $\rho$ has no node in $Z$. When there is no route in a CG $G$ between a node of $X$ and a node of $Y$ that is $Z$-active, we say that $X$ is separated from $Y$ given $Z$ in $G$ and denote it as $X \ci_G Y | Z$. We denote by $X \nci_G Y | Z$ that $X \ci_G Y | Z$ does not hold.

Let $X$, $Y$, $Z$ and $W$ denote four disjoint subsets of $V$. An independence model $M$ is a set of statements of the form $X \ci_M Y | Z$, meaning that $X$ is independent of $Y$ given $Z$. Given two independence models $M$ and $N$, we denote by $M \subseteq N$ that if $X \ci_M Y | Z$ then $X \ci_N Y | Z$. We say that $M$ is a graphoid if it satisfies the following properties: Symmetry $X \ci_M Y | Z \Rightarrow Y \ci_M X | Z$, decomposition $X \ci_M Y \cup W | Z \Rightarrow X \ci_M Y | Z$, weak union $X \ci_M Y \cup W | Z \Rightarrow X \ci_M Y | Z \cup W$, contraction $X \ci_M Y | Z \cup W \land X \ci_M W | Z \Rightarrow X \ci_M Y \cup W | Z$, and intersection $X \ci_M Y | Z \cup W \land X \ci_M W | Z \cup Y \Rightarrow X \ci_M Y \cup W | Z$. The independence model induced by a CG $G$, denoted as $I(G)$, is the set of separation statements $X \ci_G Y | Z$. It is known that $I(G)$ is a graphoid \citep[Lemma 3.1]{StudenyandBouckaert1998}. Let $H$ denote the graph resulting from a feasible split or merging in a CG $G$. Then, $H$ is a CG and $I(H)=I(G)$ \citep[Lemma 5 and Corollary 9]{Studenyetal.2009}.

A CG $G$ is an independence (I) map of an independence model $M$ if $I(G) \subseteq M$. Moreover, $G$ is a minimal independence (MI) map of $M$ if removing any edge from $G$ makes it cease to be an I map of $M$. Given any chain $C_1, \ldots, C_n$ that is consistent with $G$, we say that $G$ satisfies the pairwise block-recursive Markov property wrt $M$ if $X \ci_M Y | \cup_{j=1}^{k^*} C_j \setminus \{X, Y\}$ for all non-adjacent nodes $X$ and $Y$ of $G$ and where $k^*$ is the smallest $k$ st $X, Y \in \cup_{j=1}^{k} C_j$. If $M$ is a graphoid and $G$ satisfies the pairwise block-recursive Markov property wrt $M$, then $G$ is an I map of $M$ \citep[Theorem 3.34]{Lauritzen1996}. We say that a CG $G_{\alpha}$ is a MI map of an independence model $M$ relative to a chain $\alpha$ if $G_{\alpha}$ is a MI map of $M$ and $G_{\alpha}$ is consistent with $\alpha$.

\section{Extension of Meek's Conjecture to Chain Graphs}

Given two directed and acyclic graphs $G$ and $H$ st $I(H) \subseteq I(G)$, Meek's conjecture states that we can transform $G$ into $H$ by a sequence of arc additions and covered arc reversals st after each operation in the sequence $G$ is a directed and acyclic graph and $I(H) \subseteq I(G)$ \citep{Meek1997}. Meek's conjecture was proven to be true in \citep[Theorem 4]{Chickering2002} by developing an algorithm that constructs a valid sequence of operations. In this section, we extend Meek's conjecture from directed and acyclic graphs to CGs, and prove that the extended conjecture is true. Specifically, given two CGs $G$ and $H$ st $I(H) \subseteq I(G)$, we prove that $G$ can be transformed into $H$ by a sequence of directed and undirected edge additions and feasible splits and mergings st after each operation in the sequence $G$ is a CG and $I(H) \subseteq I(G)$. The proof is constructive in the sense that we give an algorithm that constructs a valid sequence of operations. 

\begin{figure}[t]
\centering
\small
\begin{tabular}{rl}
\hline
\\
& \underline{Fbsplit($K$, $L$, $G$)}\\
\\
& /* Given a block $K$ of a CG $G$ and a subset $L$ of $K$, the algorithm repeatedly splits a\\
& component of $G$ until $L$ becomes a block of $G$. Before the splits, the algorithm adds to $G$\\
& the smallest set of edges so that the splits are feasible */\\
\\
1 & Let $L_1, \ldots, L_n$ denote the maximal connected subsets of $L$ in $G$\\
2 & For $i=1$ to $n$ do\\
3 & \hspace{0.3cm} Add an edge $X - Y$ to $G$ for all $X, Y \in Ne_{G}(L_i) \cap (K \setminus L)$\\
4 & \hspace{0.3cm} Add an edge $X \rightarrow Y$ to $G$ for all $X \in Pa_{G}(L_i)$ and $Y \in Ne_{G}(L_i) \cap (K \setminus L)$\\
5 & For $i=1$ to $n$ do\\
6 & \hspace{0.3cm} Let $K_j$ denote the component of $G$ st $L_i \subseteq K_j$\\
7 & \hspace{0.3cm} If $K_j \setminus L_i \neq \emptyset$ then\\
8 & \hspace{0.7cm} Split $K_j$ into $K_j \setminus L_i$ and $L_i$ in $G$\\
\\
& \underline{Fbmerge($L$, $R$, $G$)}\\
\\
& /* Given two blocks $L$ and $R$ of a CG $G$, the algorithm repeatedly merges two components\\
& of $G$ until $L \cup R$ becomes a block of $G$. Before the mergings, the algorithm adds to $G$ the\\
& smallest set of edges so that the mergings are feasible */\\
\\
1 & Let $R_1, \ldots, R_n$ denote the components of $G$ that are in $R$\\
2 & For $i=1$ to $n$ do\\
3 & \hspace{0.3cm} Add an edge $X - Y$ to $G$ for all $X, Y \in Pa_{G}(R_i) \cap L$\\
4 & \hspace{0.3cm} Add an edge $X \rightarrow Y$ to $G$ for all $X \in Pa_{G}(R_i) \setminus L$ and $Y \in Pa_{G}(R_i) \cap L$\\
5 & For $i=1$ to $n$ do\\
6 & \hspace{0.3cm} Let $L_j$ denote the component of $G$ st $L_j \subseteq L \cup R$ and $Pa_G(R_i) \cap L_j \neq \emptyset$\\ 
7 & \hspace{0.3cm} If $L_j \neq \emptyset$ then\\
8 & \hspace{0.7cm} Merge $L_j$ and $R_i$ in $G$\\
\\
\hline
\\
\end{tabular}
\caption{Fbsplit and fbmerge.}\label{fig:fb}
\end{figure}

We start by introducing two new operations on CGs. It is worth mentioning that all the algorithms in this paper use a "by reference" calling convention, meaning that the algorithms can modify the arguments passed to them. Let $K$ denote a block of a CG $G$. Let $L \subseteq K$. By feasible block splitting (fbsplitting) $K$ into $K \setminus L$ and $L$ in $G$, we mean running the algorithm at the top of Figure \ref{fig:fb}. The algorithm repeatedly splits a component of $G$ until $L$ becomes a block of $G$. Before the splits, the algorithm adds to $G$ the smallest set of edges so that the splits are feasible. Let $L$ and $R$ denote two blocks of a CG $G$. By feasible block merging (fbmerging) $L$ and $R$ in $G$, we mean running the algorithm at the bottom of Figure \ref{fig:fb}. The algorithm repeatedly merges two components of $G$ until $L \cup R$ becomes a block of $G$. Before the mergings, the algorithm adds to $G$ the smallest set of edges so that the mergings are feasible. It is worth mentioning that the component $L_j$ in line 6 is guaranteed to be unique by the edges added in lines 3 and 4.

\begin{figure}[t]
\centering
\small
\begin{tabular}{rl}
\hline
\\
& \underline{Construct $\beta$($G$, $\alpha$, $\beta$)}\\
\\
& /* Given a CG $G$ and a chain $\alpha$, the algorithm derives a chain $\beta$ that is consistent with $G$\\
& and as close to $\alpha$ as possible */\\
\\
1 & Set $\beta=\emptyset$\\
2 & Set $H=G$\\
3 & Let $C$ denote any terminal component of $H$ whose leftmost node in $\alpha$ is rightmost in $\alpha$\\
4 & Add $C$ as the leftmost block of $\beta$\\
5 & Let $R$ denote the right neighbor of $C$ in $\beta$\\
6 & If $R \neq \emptyset$, $Pa_G(R) \cap C = \emptyset$, and the nodes of $C$ are to the right of the nodes of $R$ in $\alpha$ then\\
7 & \hspace{0.3cm} Replace $C, R$ with $R, C$ in $\beta$\\
8 & \hspace{0.3cm} Go to line 5\\
9 & Remove $C$ and all its incoming edges from $H$\\
10 & If $H \neq \emptyset$ then\\
11 & \hspace{0.3cm} Go to line 3\\
\\
& \underline{Method B3($G$, $\alpha$)}\\
\\
& /* Given a CG $G$ and a chain $\alpha$, the algorithm transforms $G$ into $G_{\alpha}$ */\\
\\
1 & Construct $\beta$($G$, $\alpha$, $\beta$)\\
2 & Let $C$ denote the rightmost block of $\alpha$ that has not been considered before\\
3 & Let $K$ denote the leftmost block of $\beta$ st $K \cap C \neq \emptyset$\\
4 & Set $L = K \cap C$\\
5 & If $K \setminus L \neq \emptyset$ then\\
6 & \hspace{0.3cm} Fbsplit($K$, $L$, $G$)\\
7 & \hspace{0.3cm} Replace $K$ with $K \setminus L, L$ in $\beta$\\
8 & Let $R$ denote the right neighbor of $L$ in $\beta$\\
9 & If $R \neq \emptyset$ and some node of $R$ is not to the right of the nodes of $L$ in $\alpha$\\
10 & \hspace{0.3cm} Fbmerge($L$, $R$, $G$)\\
11 & \hspace{0.3cm} Replace $L, R$ with $L \cup R$ in $\beta$\\
12 & \hspace{0.3cm} Go to line 3\\
13 & If $\beta \neq \alpha$ then\\
14 & \hspace{0.3cm} Go to line 2\\
\\
\hline
\\
\end{tabular}
\caption{Method B3.}\label{fig:methodb2}
\end{figure}

Our proof of the extension of Meek's conjecture to CGs builds upon an algorithm for efficiently deriving the MI map $G_{\alpha}$ of the independence model induced by a given CG $G$ relative to a given chain $\alpha$. The pseudocode of the algorithm, called Method B3, can be seen in Figure \ref{fig:methodb2}. Method B3 works iteratively by fbsplitting and fbmerging some blocks of $G$ until the resulting CG is consistent with $\alpha$. It is not difficult to see that such a way of working results in a CG that is an I map of $I(G)$. However, in order to arrive at $G_{\alpha}$, the blocks of $G$ to modify in each iteration must be carefully chosen. For this purpose, Method B3 starts by calling Construct $\beta$ to derive a chain $\beta$ that is consistent with $G$ and as close to $\alpha$ as possible (see lines 5-8). By $\beta$ being as close to $\alpha$ as possible, we mean that the number of blocks Method B3 will later fbsplit and fbmerge is kept at a minimum, because Method B3 will use $\beta$ to choose the blocks to modify in each iteration. A line of Construct $\beta$ that is worth explaining is line 3, because it is crucial for the correctness of Method B3 (see Case 3.2.4 in the proof of Lemma \ref{lem:correctness}). This line determines the order in which the components of $H$ (initially $H=G$) are added to $\beta$ (initially $\beta = \emptyset$). In principle, a component of $H$ may have nodes from several blocks of $\alpha$. Line 3 labels each terminal component of $H$ with its leftmost node in $\alpha$ and, then, chooses any terminal component whose label node is rightmost in $\alpha$. This is the next component to add to $\beta$.

Once $\beta$ has been constructed, Method B3 proceeds to transform $G$ into $G_{\alpha}$. In particular, Method B3 considers the blocks of $\alpha$ one by one in the reverse order in which they appear in $\alpha$. For each block $C$ of $\alpha$, Method B3 iterates through the following steps. First, it finds the leftmost block $K$ of $\beta$ that has some nodes from $C$. These nodes, denoted as $L$, are then moved to the right in $\beta$ by fbsplitting $K$ to create a new block $L$ of $G$ and $\beta$. If the nodes of the right neighbor $R$ of $L$ in $\beta$ are to the right of the nodes of $L$ in $\alpha$, then Method B3 is done with $C$. Otherwise, Method B3 moves $L$ further to the right in $\beta$ by fbmerging $L$ and $R$ in $G$ and $\beta$. We prove below that Method B3 is correct. We prove first some auxiliary results.

\begin{lem}\label{lem:unique}
Let $M$ denote an independence model, and $\alpha$ a chain $C_1, \ldots, C_n$. If $M$ is a graphoid, then there exits a unique CG $G_{\alpha}$ that is a MI map of $M$ relative to $\alpha$. Specifically, for each node $X$ of each block $C_k$ of $\alpha$, $Bd_{G_{\alpha}}(X)$ is the smallest subset $B$ of $\cup_{j=1}^k C_j \setminus \{X\}$ st $X \ci_M \cup_{j=1}^k C_j \setminus \{X\} \setminus B  | B$.\footnote{By convention, $X \ci_M \emptyset | \cup_{j=1}^k C_j \setminus \{X\}$.}
\end{lem}

\begin{proof}

Let $X$ and $Y$ denote any two non-adjacent nodes of $G_{\alpha}$. Let $k^*$ denote the smallest $k$ st $X, Y \in \cup_{j=1}^{k} C_j$. Assume without loss of generality that $X \in C_{k^*}$. Then, $X \ci_M \cup_{j=1}^{k^*} C_j \setminus \{X\} \setminus Bd_{G_{\alpha}}(X) | Bd_{G_{\alpha}}(X)$ by construction of $G_{\alpha}$ and, thus, $X \ci_M Y | \cup_{j=1}^{k^*} C_j \setminus \{X, Y\}$ by weak union. Then, $G_{\alpha}$ satisfies the pairwise block-recursive Markov property wrt $M$ and, thus, $G_{\alpha}$ is an I map of $M$. In fact, $G_{\alpha}$ is a MI map of $M$ by construction of $Bd_{G_{\alpha}}(X)$.

Assume to the contrary that there exists another CG $H_{\alpha}$ that is a MI map of $M$ relative to $\alpha$. Let $X$ denote any node st $Bd_{G_{\alpha}}(X) \neq Bd_{H_{\alpha}}(X)$. Let $X \in C_k$. Then, $X \ci_M \cup_{j=1}^k C_j \setminus \{X\} \setminus Bd_{G_{\alpha}}(X) | Bd_{G_{\alpha}}(X)$ and $X \ci_M \cup_{j=1}^k C_j \setminus \{X\} \setminus Bd_{H_{\alpha}}(X) | Bd_{H_{\alpha}}(X)$ because $G_{\alpha}$ and $H_{\alpha}$ are MI maps of $M$. Then, $X \ci_M \cup_{j=1}^k C_j \setminus \{X\} \setminus Bd_{G_{\alpha}}(X) \cap Bd_{H_{\alpha}}(X) | Bd_{G_{\alpha}}(X) \cap Bd_{H_{\alpha}}(X)$ by intersection. However, this contradicts the construction of $Bd_{G_{\alpha}}(X)$, because $Bd_{G_{\alpha}}(X) \cap Bd_{H_{\alpha}}(X)$ is smaller than $Bd_{G_{\alpha}}(X)$.

\end{proof}

\begin{lem}\label{lem:maximal}
Let $G$ and $H$ denote two CGs st $I(H) \subseteq I(G)$. For any component $C$ of $G$, there exists a unique component of $H$ that is maximal in $H$ from the set of components of $H$ that contain a descendant of $C$ in $G$.
\end{lem}

\begin{proof}

By definition of CG, there exists at least one such component of $H$. Assume to the contrary that there exist two such components of $H$, say $K$ and $K'$. Note that $Pa_H(K) \cap K' = \emptyset$ and $Pa_H(K') \cap K = \emptyset$ by definition of $K$ and $K'$. Note also that no node of $K$ or $Pa_H(K)$ is a descendant of $K'$ in $H$ by definition of $K$. This implies that $K' \ci_H K \cup Pa_H(K) \setminus Pa_H(K') | Pa_H(K')$ and, thus, $K \ci_H K' | Pa_H(K) \cup Pa_H(K')$ by weak union and symmetry.

That $K$ and $K'$ contain some descendants $k$ and $k'$ of $C$ in $G$ implies that there are descending routes from $C$ to $k$ and $k'$ in $G$ st the nodes in the routes are descendant of $C$ in $G$. Thus, there is a route between $k$ and $k'$ in $G$ st the nodes in the route are descendant of $C$ in $G$. Note that no node in this route is in $Pa_H(K)$ or $Pa_H(K')$ by definition of $K$ and $K'$. Then, $K \nci_G K' | Pa_H(K) \cup Pa_H(K')$. However, this contradicts the fact that $I(H) \subseteq I(G)$ because, as shown, $K \ci_H K' | Pa_H(K) \cup Pa_H(K')$.

\end{proof}

\begin{lem}\label{lem:descendants}
Let $G$ and $H$ denote two CGs st $I(H) \subseteq I(G)$. Let $\alpha$ denote a chain that is consistent with $H$. If no descendant of a node $X$ in $G$ is to the left of $X$ in $\alpha$, then the descendants of $X$ in $G$ are descendant of $X$ in $H$ too.
\end{lem}

\begin{proof}

Let $D$ denote the descendants of $X$ in $G$. Let $C$ denote the component of $G$ that contains $X$. Note that the descendants of $C$ in $G$ are exactly the set $D$. Then, there exists a unique component of $H$ that is maximal in $H$ from the set of components of $H$ that contain a node from $D$, by Lemma \ref{lem:maximal}.

Let $K$ denote the component of $H$ that contains $X$. Note that $K$ is a component of $H$ that is maximal in $H$ from the set of components of $H$ that contain a node from $D$, since no node of $D$ is to the left of $X$ in $\alpha$. It follows from the paragraph above that $K$ is the only such component of $H$.

\end{proof}

We are now ready to prove the correctness of Method B3.

\begin{lem}\label{lem:correctness}
Let $G_{\alpha}$ denote the MI map of the independence model induced by a CG $G$ relative to a chain $\alpha$. Then, Method B3($G$, $\alpha$) returns $G_{\alpha}$.
\end{lem}

\begin{proof}

We start by proving that Method B3 halts at some point. When Method B3 is done with the rightmost block of $\alpha$, the rightmost block of $\beta$ contains all and only the nodes of the rightmost block of $\alpha$. When Method B3 is done with the second rightmost block of $\alpha$, the rightmost block of $\beta$ contains all and only the nodes of the rightmost block of $\alpha$, whereas the second rightmost block of $\beta$ contains all and only the nodes of the second rightmost block of $\alpha$. Continuing with this reasoning, one can see that when Method B3 is done with all the blocks of $\alpha$, $\beta$ coincides with $\alpha$ and thus Method B3 halts.

That Method B3 halts at some point implies that it performs a finite sequence of $m$ modifications to $G$ due to the fbsplit and fbmerging in lines 6 and 10. Let $G_t$ denote the CG resulting from the first $t$ modifications to $G$, and let $G_0=G$. Specifically, Method B3 constructs $G_{t+1}$ from $G_t$ by either 

\begin{itemize}

\item adding an edge $X - Y$ due to line 3 of Fbsplit or Fbmerge,

\item adding an edge $X \rightarrow Y$ due to line 4 of Fbsplit or Fbmerge,

\item performing all the component splits due to lines 5-8 of Fbsplit, or

\item performing all the component mergings due to lines 5-8 of Fbmerge.

\end{itemize}

Note that none of the modifications above introduces new separation statements. This is trivial to see for the first and second modification. To see it for the third and fourth modification, recall that the splits and the mergings are part of a fbsplit and a fbmerging respectively and, thus, they are feasible. Therefore, $I(G_{t+1}) \subseteq I(G_{t})$ for all $0 \leq t < m$ and, thus, $I(G_{m}) \subseteq I(G_{0})$.

We continue by proving that $G_t$ is consistent with $\beta$ for all $0 \leq t \leq m$. Since this is true for $G_0$ due to line 1, it suffices to prove that if it is true for $G_t$ then it is true for $G_{t+1}$ for all $0 \leq t < m$. We consider the following four cases.

\begin{description}

\item[Case 1] Method B3 constructs $G_{t+1}$ from $G_t$ by adding an edge $X - Y$ due to line 3 of Fbsplit or Fbmerge. It suffices to note that $X$ and $Y$ are in the same block of $G_t$ and $\beta$.

\item[Case 2] Method B3 constructs $G_{t+1}$ from $G_t$ by adding an edge $X \rightarrow Y$ due to line 4 of Fbsplit. It suffices to note that $X$ is to the left of $Y$ in $\beta$, because $G_t$ is consistent with $\beta$.

\item[Case 3] Method B3 constructs $G_{t+1}$ from $G_t$ by adding an edge $X \rightarrow Y$ due to line 4 of Fbmerge. Note that $X$ is to the left of $R$ in $\beta$, because $\beta$ is consistent with $G_t$. Then, $X$ is to the left of $L$ in $\beta$, because $L$ is the left neighbor of $R$ in $\beta$ and $X \notin L$. Then, $X$ is to the left of $Y$ in $\beta$, because $Y \in L$.

\item[Case 4] Method B3 constructs $G_{t+1}$ from $G_t$ by either performing all the component splits due to lines 5-8 of Fbsplit or performing all the component mergings due to lines 5-8 of Fbmerge. Note that the splits and the mergings are feasible, since they are part of a fbsplit and a fbmerging respectively. Therefore, $G_{t+1}$ is a CG. Moreover, note that $\beta$ is modified immediately after the fbsplit and the fbmerging so that it is consistent with $G_{t+1}$.

\end{description}

Note that $G_m$ is not only consistent with $\beta$ but also with $\alpha$ because, as shown, $\beta$ coincides with $\alpha$ when Method B3 halts. In order to prove the lemma, i.e. that $G_m=G_{\alpha}$, all that remains to prove is that $I(G_{\alpha}) \subseteq I(G_m)$. To see it, note that $G_m=G_{\alpha}$ follows from $I(G_{\alpha}) \subseteq I(G_m)$, $I(G_m) \subseteq I(G_0)$, the fact that $G_m$ is consistent with $\alpha$, and the fact that $G_{\alpha}$ is the unique MI map of $I(G_0)$ relative to $\alpha$. Recall that $G_{\alpha}$ is guaranteed to be unique by Lemma \ref{lem:unique}, because $I(G_0)$ is a graphoid.

The rest of the proof is devoted to prove that $I(G_{\alpha}) \subseteq I(G_m)$. Specifically, we prove that if $I(G_{\alpha}) \subseteq I(G_t)$ then $I(G_{\alpha}) \subseteq I(G_{t+1})$ for all $0 \leq t < m$. Note that this implies that $I(G_{\alpha}) \subseteq I(G_m)$ because $I(G_{\alpha}) \subseteq I(G_0)$ by definition of MI map. First, we prove it when Method B3 constructs $G_{t+1}$ from $G_t$ by either performing all the component splits due to lines 5-8 of Fbsplit or performing all the component mergings due to lines 5-8 of Fbmerge. Note that the splits and the mergings are feasible, since they are part of a fbsplit and a fbmerging respectively. Therefore, $I(G_{t+1})=I(G_t)$. Thus, $I(G_{\alpha}) \subseteq I(G_{t+1})$ because $I(G_{\alpha}) \subseteq I(G_t)$.

Now, we prove that if $I(G_{\alpha}) \subseteq I(G_t)$ then $I(G_{\alpha}) \subseteq I(G_{t+1})$ when Method B3 constructs $G_{t+1}$ from $G_t$ by adding a directed or undirected edge due to lines 3 and 4 of Fbsplit and Fbmerge. Specifically, we prove that if there is an $S$-active route $\rho^{AB}_{t+1}$ between two nodes $A$ and $B$ in $G_{t+1}$, then there is an $S$-active route between $A$ and $B$ in $G_{\alpha}$. We prove this result by induction on the number of occurrences of the added edge in $\rho^{AB}_{t+1}$. We assume without loss of generality that the added edge occurs in $\rho^{AB}_{t+1}$ as few or fewer times than in any other $S$-active route between $A$ and $B$ in $G_{t+1}$. We call this the minimality property of $\rho^{AB}_{t+1}$. If the number of occurrences of the added edge in $\rho^{AB}_{t+1}$ is zero, then $\rho^{AB}_{t+1}$ is an $S$-active route between $A$ and $B$ in $G_t$ too and, thus, there is an $S$-active route between $A$ and $B$ in $G_{\alpha}$ since $I(G_{\alpha}) \subseteq I(G_t)$. Assume as induction hypothesis that the result holds for up to $n$ occurrences of the added edge in $\rho^{AB}_{t+1}$. We now prove it for $n+1$ occurrences. We consider the following four cases.

\begin{description}

\item[Case 1] Method B3 constructs $G_{t+1}$ from $G_t$ by adding an edge $X - Y$ due to line 3 of Fbsplit. Note that $X - Y$ occurs in $\rho^{AB}_{t+1}$.\footnote{Note that maybe $A=X$ and/or $Y=B$.} Assume that $X - Y$ occurs in a collider section of $\rho^{AB}_{t+1}$. Note that $X$ and $Y$ must be in the same component of $G_t$ for line 3 of Fbsplit to add an edge $X - Y$. This component also contains a node $Z$ that is in $S$ because, otherwise, $\rho^{AB}_{t+1}$ would not be $S$-active in $G_{t+1}$.\footnote{Note that maybe $Z=X$ or $Z=Y$.} Note that there is a route $X - \ldots - Z - \ldots - Y$ in $G_t$. Then, we can replace any occurrence of $X - Y$ in a collider section of $\rho^{AB}_{t+1}$ with $X - \ldots - Z - \ldots - Y$, and thus construct an $S$-active route between $A$ and $B$ in $G_{t+1}$ that violates the minimality property of $\rho^{AB}_{t+1}$. Since this is a contradiction, $X - Y$ only occurs in non-collider sections of $\rho^{AB}_{t+1}$. Let $\rho^{AB}_{t+1} = \rho^{AX}_{t+1} \cup X - Y \cup \rho^{YB}_{t+1}$. Note that $X, Y \notin S$ because, otherwise, $\rho^{AB}_{t+1}$ would not be $S$-active in $G_{t+1}$. For the same reason, $\rho^{AX}_{t+1}$ and $\rho^{YB}_{t+1}$ are $S$-active in $G_{t+1}$. Then, there are $S$-active routes $\rho^{AX}_{\alpha}$ and $\rho^{YB}_{\alpha}$ between $A$ and $X$ and between $Y$ and $B$ in $G_{\alpha}$ by the induction hypothesis. 

Let $X - X' - \ldots - Y' - Y$ be a route in $G_t$ st the nodes in $X' - \ldots - Y'$ are in $L$.\footnote{Note that maybe $X'=Y'$.} Such a route must exist for line 3 of Fbsplit to add an edge $X - Y$. Note that $X$ and $X'$ are adjacent in $G_{\alpha}$ since $I(G_{\alpha}) \subseteq I(G_t)$. In fact, $X \rightarrow X'$ is in $G_{\alpha}$. To see it, recall that Method B3 is currently considering the block $C$ of $\alpha$, and that it has previously considered all the blocks of $\alpha$ to the right of $C$ in $\alpha$. Then, $K$ only contains nodes from $C$ or from blocks to the left of $C$ in $\alpha$. However, $X \notin C$ because $X \in K \setminus L$ and $L=K \cap C$. Then, $X$ is to the left of $C$ in $\alpha$. Thus, $X \rightarrow X'$ is in $G_{\alpha}$ because $X' \in L \subseteq C$. Likewise, $Y \rightarrow Y'$ is in $G_{\alpha}$. Note also that $X' - \ldots - Y'$ is in $G_{\alpha}$. To see it, note that the adjacencies in $X' - \ldots - Y'$ are preserved in $G_{\alpha}$ since $I(G_{\alpha}) \subseteq I(G_t)$. Moreover, these adjacencies correspond to undirected edges in $G_{\alpha}$, because the nodes in $X' - \ldots - Y'$ are in $L$ and thus in the same block of $G_{\alpha}$, since $L \subseteq C$. Furthermore, a node in $X' - \ldots - Y'$ is in $S$ because, otherwise, $\rho^{AX}_{t+1} \cup X - X' - \ldots - Y' - Y \cup \rho^{YB}_{t+1}$ would be an $S$-active route between $A$ and $B$ in $G_{t+1}$ that would violate the minimality property of $\rho^{AB}_{t+1}$. Then, $\rho^{AX}_{\alpha} \cup X \rightarrow X' - \ldots - Y' \leftarrow Y \cup \rho^{YB}_{\alpha}$ is an $S$-active route between $A$ and $B$ in $G_{\alpha}$.

\item[Case 2] Method B3 constructs $G_{t+1}$ from $G_t$ by adding an edge $X \rightarrow Y$ due to line 4 of Fbsplit. Note that $X \rightarrow Y$ occurs in $\rho^{AB}_{t+1}$.\footnote{Note that maybe $A=X$ and/or $Y=B$.} Assume that $X \rightarrow Y$ occurs as a collider edge in $\rho^{AB}_{t+1}$, i.e. $X \rightarrow Y$ occurs in a subroute of $\rho^{AB}_{t+1}$ of the form $X \rightarrow Y - \ldots - Z \leftarrow W$.\footnote{Note that maybe $Y=Z$ and/or $W=X$.} Note that a node in $Y - \ldots - Z$ is in $S$ because, otherwise, $\rho^{AB}_{t+1}$ would not be $S$-active in $G_{t+1}$. Let $X \rightarrow X' - \ldots - Y' - Y$ be a route in $G_t$ st the nodes in $X' - \ldots - Y'$ are in $L$.\footnote{Note that maybe $X'=Y'$.} Such a route must exist for line 4 of Fbsplit to add an edge $X \rightarrow Y$. Then, we can replace $X \rightarrow Y - \ldots - Z \leftarrow W$ with $X \rightarrow X' - \ldots - Y' - Y - \ldots - Z \leftarrow W$ in $\rho^{AB}_{t+1}$, and thus construct an $S$-active route between $A$ and $B$ in $G_{t+1}$ that violates the minimality property of $\rho^{AB}_{t+1}$. Since this is a contradiction, $X \rightarrow Y$ never occurs as a collider edge in $\rho^{AB}_{t+1}$. Let $\rho^{AB}_{t+1} = \rho^{AX}_{t+1} \cup X \rightarrow Y \cup \rho^{YB}_{t+1}$. Note that $X, Y \notin S$ because, otherwise, $\rho^{AB}_{t+1}$ would not be $S$-active in $G_{t+1}$. For the same reason, $\rho^{AX}_{t+1}$ and $\rho^{YB}_{t+1}$ are $S$-active in $G_{t+1}$. Then, there are $S$-active routes $\rho^{AX}_{\alpha}$ and $\rho^{YB}_{\alpha}$ between $A$ and $X$ and between $Y$ and $B$ in $G_{\alpha}$ by the induction hypothesis.

Let $X \rightarrow X' - \ldots - Y' - Y$ denote a route in $G_t$ st the nodes in $X' - \ldots - Y'$ are in $L$.\footnote{Note that maybe $X'=Y'$.} Such a route must exist for line 4 of Fbsplit to add an edge $X \rightarrow Y$. Note that $X' - \ldots - Y'$ is in $G_{\alpha}$. To see it, note that the adjacencies in $X' - \ldots - Y'$ are preserved in $G_{\alpha}$ since $I(G_{\alpha}) \subseteq I(G_t)$. Moreover, these adjacencies correspond to undirected edges in $G_{\alpha}$, because the nodes in $X' - \ldots - Y'$ are in $L$ and thus in the same block of $G_{\alpha}$, since $L \subseteq C$. Furthermore, a node in $X' - \ldots - Y'$ is in $S$ because, otherwise, $\rho^{AX}_{t+1} \cup X \rightarrow X' - \ldots - Y' - Y \cup \rho^{YB}_{t+1}$ would be an $S$-active route between $A$ and $B$ in $G_{t+1}$ that would violate the minimality property of $\rho^{AB}_{t+1}$. Moreover, note that $X$ and $X'$ are adjacent in $G_{\alpha}$ since $I(G_{\alpha}) \subseteq I(G_t)$. In fact, $X \rightarrow X'$ is in $G_{\alpha}$. To see it, recall that Method B3 is currently considering the block $C$ of $\alpha$, and that it has previously considered all the blocks of $\alpha$ to the right of $C$ in $\alpha$. Then, no block to the left of $K$ in $\beta$ has a node from $C$ or from a block to the right of $C$ in $\alpha$. Note that $X$ is to the left of $K$ in $\beta$, because $\beta$ is consistent with $G_t$. Thus, $X \rightarrow X'$ is in $G_{\alpha}$ since $X' \in L \subseteq C$. Likewise, note that $Y'$ and $Y$ are adjacent in $G_{\alpha}$ since $I(G_{\alpha}) \subseteq I(G_t)$. In fact, $Y' \leftarrow Y$ is in $G_{\alpha}$. To see it, note that $K$ only contains nodes from $C$ or from blocks to the left of $C$ in $\alpha$. However, $Y \notin C$ because $Y \in K \setminus L$ and $L=K \cap C$. Then, $Y$ is to the left of $C$ in $\alpha$. Thus, $Y' \leftarrow Y$ is in $G_{\alpha}$ because $Y' \in L \subseteq C$. Then, $\rho^{AX}_{\alpha} \cup X \rightarrow X' - \ldots - Y' \leftarrow Y \cup \rho^{YB}_{\alpha}$ is an $S$-active route between $A$ and $B$ in $G_{\alpha}$.

\item[Case 3] Method B3 constructs $G_{t+1}$ from $G_t$ by adding an edge $X - Y$ due to line 3 of Fbmerge. Note that $X - Y$ occurs in $\rho^{AB}_{t+1}$. We consider two cases.

\begin{description}

\item[Case 3.1] Assume that $X - Y$ occurs in a collider section of $\rho^{AB}_{t+1}$. Let $\rho^{AB}_{t+1} = \rho^{AZ}_{t+1} \cup Z \rightarrow X' - \ldots - X - Y - \ldots - Y' \leftarrow W \cup \rho^{WB}_{t+1}$.\footnote{Note that maybe $A=Z$, $X'=X$, $Y'=Y$, $W=Z$ and/or $W=B$.} Note that $Z, W \notin S$ because, otherwise, $\rho^{AB}_{t+1}$ would not be $S$-active in $G_{t+1}$. For the same reason, $\rho^{AZ}_{t+1}$ and $\rho^{WB}_{t+1}$ are $S$-active in $G_{t+1}$. Then, there are $S$-active routes $\rho^{AZ}_{\alpha}$ and $\rho^{WB}_{\alpha}$ between $A$ and $Z$ and between $W$ and $B$ in $G_{\alpha}$ by the induction hypothesis.

Let $R_i$ denote the component of $G_t$ in $R$ that Fbmerge is processing when the edge $X - Y$ gets added. Recall that Method B3 is currently considering the block $C$ of $\alpha$, and that it has previously considered all the blocks of $\alpha$ to the right of $C$ in $\alpha$. Then, $R_i$ only contains nodes from $C$ or from blocks to the left of $C$ in $\alpha$. In other words, $R_i \subseteq \cup_{j=1}^{k^*} C_j \setminus \{X, Y\}$ where $C_{k^*} = C$ (recall that $X, Y \in L \subseteq C$). Therefore, $X \nci_{G_t} Y | \cup_{j=1}^{k^*} C_j \setminus \{X, Y\}$ because $X$ and $Y$ must be in $Pa_{G_t}(R_i)$ for line 3 of Fbmerge to add an edge $X - Y$. Then, $X$ and $Y$ are adjacent in $G_{\alpha}$ because, otherwise, $X \ci_{G_{\alpha}} Y | \cup_{j=1}^{k^*} C_j \setminus \{X, Y\}$ which would contradict that $I(G_{\alpha}) \subseteq I(G_t)$. In fact, $X - Y$ is in $G_{\alpha}$ because $X$ and $Y$ are in the same block of $\alpha$, since $X, Y \in L \subseteq C$.

Note that $X' - \ldots - X$ and $Y - \ldots - Y'$ are in $G_{\alpha}$. To see it, note that the adjacencies in $X' - \ldots - X$ and $Y - \ldots - Y'$ are preserved in $G_{\alpha}$ since $I(G_{\alpha}) \subseteq I(G_t)$. Moreover, these adjacencies correspond to undirected edges in $G_{\alpha}$, because the nodes in $X' - \ldots - X$ and $Y - \ldots - Y'$ are in $L$ since $X, Y \in L$ and, thus, they are in the same block of $G_{\alpha}$ since $L \subseteq C$. Then, $X' - \ldots - X - Y - \ldots - Y'$ is in $G_{\alpha}$. Furthermore, a node in $X' - \ldots - X - Y - \ldots - Y'$ is in $S$ because, otherwise, $\rho^{AB}_{t+1}$ would not be $S$-active in $G_{t+1}$. Note also that $Z$ and $X'$ are adjacent in $G_{\alpha}$ since $I(G_{\alpha}) \subseteq I(G_t)$. In fact, $Z \rightarrow X'$ is in $G_{\alpha}$. To see it, recall that Method B3 is currently considering the block $C$ of $\alpha$, and that it has previously considered all the blocks of $\alpha$ to the right of $C$ in $\alpha$. Then, no block to the left of $L$ in $\beta$ has a node from $C$ or from a block to the right of $C$ in $\alpha$. Note that $Z$ is to the left of $L$ in $\beta$, because $\beta$ is consistent with $G_t$. Thus, $Z \rightarrow X'$ is in $G_{\alpha}$ since $X' \in L \subseteq C$. Likewise, $Y' \leftarrow W$ is in $G_{\alpha}$. Then, $\rho^{AZ}_{\alpha} \cup Z \rightarrow X' - \ldots - X - Y - \ldots - Y' \leftarrow W \cup \rho^{WB}_{\alpha}$ is an $S$-active route between $A$ and $B$ in $G_{\alpha}$.

\item[Case 3.2] Assume that $X - Y$ occurs in a non-collider section of $\rho^{AB}_{t+1}$. Note that this implies that $G_t$ has a descending route from $X$ to $A$ or to a node in $S$, or from $Y$ to $B$ or to a node in $S$. Assume without loss of generality that $G_t$ has a descending route from $Y$ to $B$ or to a node in $S$.

Let $R_i$ denote the component of $G_t$ in $R$ that Fbmerge is processing when the edge $X - Y$ gets added. Let $L_Y$ denote the component of $G_t$ that contains the node $Y$. Let $D$ denote the component of $G_{\alpha}$ that is maximal in $G_{\alpha}$ from the set of components of $G_{\alpha}$ that contain a descendant of $L_Y$ in $G_t$. Recall that $D$ is guaranteed to be unique by Lemma \ref{lem:maximal}, because $I(G_{\alpha}) \subseteq I(G_t)$. We now show that some $d \in D$ is a descendant of $R_i$ in $G_t$. We consider four cases.

\begin{description}

\item[Case 3.2.1] Assume that $D \cap L_Y \neq \emptyset$. It suffices to consider any $d \in R_i$. To see it, recall that Method B3 is currently considering the block $C$ of $\alpha$, and that it has previously considered all the blocks of $\alpha$ to the right of $C$ in $\alpha$. Then, $R_i$ only contains nodes from $C$ or from blocks to the left of $C$ in $\alpha$. Thus, $d$ is not to the right of the nodes of $D \cap L_Y$ in $\alpha$, since $L_Y \subseteq L \subseteq C$. Moreover, $d$ is not to the left of the nodes of $D \cap L_Y$ in $\alpha$ because, otherwise, there would be a contradiction with the definition of $D$. Then, $d \in D$.

\item[Case 3.2.2] Assume that $D \cap L_Y = \emptyset$ and $D \cap R_i \neq \emptyset$. It suffices to consider any $d \in D \cap R_i$.

\item[Case 3.2.3] Assume that $D \cap L_Y = \emptyset$, $D \cap R_i = \emptyset$, and some $d \in D$ was a descendant of some $r \in R_i$ in $G_0$. Recall that Method B3 is currently considering the block $C$ of $\alpha$, and that it has previously considered all the blocks of $\alpha$ to the right of $C$ in $\alpha$. Then, $R_i$ only contains nodes from $C$ or from blocks to the left of $C$ in $\alpha$. Then, $r$ was not in the blocks of $\alpha$ previously considered, since $r \in R_i$. Therefore, no descendant of $r$ in $G_0$ is currently to the left of $r$ in $\beta$ and, thus, the descendants of $r$ in $G_0$ are descendant of $r$ in $G_t$ by Lemma \ref{lem:descendants}, because $I(G_t) \subseteq I(G_0)$ and $\beta$ is consistent with $G_t$. Then, $d$ is a descendant of $r$ and thus of $R_i$ in $G_t$.

\item[Case 3.2.4] Assume that $D \cap L_Y = \emptyset$, $D \cap R_i = \emptyset$, and no node of $D$ was a descendant of a node of $R_i$ in $G_0$. As shown in Case 3.2.3, the descendants of any node $r \in R_i$ in $G_0$ are descendant of $r$ in $G_t$ too. Therefore, no descendant of $r$ in $G_0$ was to the left of the nodes of $D$ in $\alpha$ because, otherwise, a descendant of $r$ and thus of $L_Y$ in $G_t$ would be to the left of the nodes of $D$ in $\alpha$, which would contradict the definition of $D$. Recall that no descendant of $r$ in $G_0$ was in $D$ either. Note also that the nodes of $D$ are to the left of the nodes of $R_i$ in $\alpha$, by definition of $D$ and the fact that $D \cap R_i = \emptyset$. These observations have two consequences. First, the components of $G$ containing a node from $D$ were still in $H$ when any component of $G$ containing a node from $R_i$ became a terminal component of $H$ in Construct $\beta$. Thus, Construct $\beta$ added the components of $G$ containing a node from $D$ to $\beta$ after having added the components of $G$ containing a node from $R_i$. Second, Construct $\beta$ did not interchange in $\beta$ any component of $G$ containing a node from $D$ with any component of $G$ containing a node from $R_i$. 

Recall that Method B3 is currently considering the block $C$ of $\alpha$, and that it has previously considered all the blocks of $\alpha$ to the right of $C$ in $\alpha$. Note that the nodes of $D$ were not in the blocks of $\alpha$ previously considered because, otherwise, $C$ and thus the nodes of $L_Y$ (recall that $L_Y \subseteq L \subseteq C$) would be to the left of $D$ in $\alpha$, which would contradict the definition of $D$. Therefore, the nodes of $D$ are currently still to the left of $R_i$ in $\beta$. Note that the only component to the left of $R_i$ in $\beta$ that contains a descendant of $L_Y$ in $G_t$ is precisely $L_Y$, because $L$ is the left neighbor of $R$ in $\beta$, $L_Y \subseteq L$, and $\beta$ is consistent with $G_t$. However, $D \cap L_Y = \emptyset$. Thus, $D$ contains no descendant of $L_Y$ in $G_t$, which contradicts the definition of $D$. Thus, this case never occurs.

\end{description}

We continue with the proof of Case 3.2. Let $\rho^{AB}_{t+1} = \rho^{AX}_{t+1} \cup X - Y \cup \rho^{YB}_{t+1}$.\footnote{Note that maybe $A=X$ and/or $Y=B$.} Note that $X, Y \notin S$ because, otherwise, $\rho^{AB}_{t+1}$ would not be $S$-active in $G_{t+1}$. For the same reason, $\rho^{AX}_{t+1}$ and $\rho^{YB}_{t+1}$ are $S$-active in $G_{t+1}$. Note that $X$ and $Y$ must be in $Pa_{G_t}(R_i)$ for line 3 of Fbmerge to add an edge $X - Y$. Then, no descendant of $R_i$ in $G_t$ is in $S$ because, otherwise, there would be an $S$-active route $\rho^{XY}_t$ between $X$ and $Y$ in $G_t$ and, thus, $\rho^{AX}_{t+1} \cup \rho^{XY}_t \cup \rho^{YB}_{t+1}$ would be an $S$-active route between $A$ and $B$ in $G_{t+1}$ that would violate the minimality property of $\rho^{AB}_{t+1}$. Then, there is an $S$-active descending route $\rho^{rd}_t$ from some $r \in R_i$ to some $d \in D$ in $G_t$ because, as shown, $D$ contains a descendant of $R_i$ in $G_t$. Then, $\rho^{AX}_{t+1} \cup X \rightarrow X' - \ldots - r \cup \rho^{rd}_t$ is an $S$-active route between $A$ and $d$ in $G_{t+1}$.\footnote{Note that maybe $X'=r$.} Likewise, $\rho^{BY}_{t+1} \cup Y \rightarrow Y' - \ldots - r \cup \rho^{rd}_t$ is an $S$-active route between $B$ and $d$ in $G_{t+1}$, where $\rho^{BY}_{t+1}$ denotes the route resulting from reversing $\rho^{YB}_{t+1}$.\footnote{Note that maybe $Y'=r$.} Therefore, there are $S$-active routes $\rho^{Ad}_{\alpha}$ and $\rho^{Bd}_{\alpha}$ between $A$ and $d$ and between $B$ and $d$ in $G_{\alpha}$ by the induction hypothesis. 

Recall that we assumed without loss of generality that $G_t$ has a descending route from $Y$ to a node $E$ st $E=B$ or $E \in S$. Note that $E$ is a descendant of $L_Y$ in $G_t$ and, thus, $E$ is a descendant of $d$ in $G_{\alpha}$ by definition of $D$ and the fact that $d \in D$. Let $\rho^{dE}_{\alpha}$ denote the descending route from $d$ to $E$ in $G_{\alpha}$. Assume without loss of generality that $G_{\alpha}$ has no descending route from $d$ to $B$ or to a node of $S$ that is shorter than $\rho^{dE}_{\alpha}$. We now consider two cases.

\begin{description}

\item[Case 3.2.5] Assume that $E=B$. Note that $\rho^{dE}_{\alpha}$ is $S$-active in $G_{\alpha}$ by definition and the fact that $d \notin S$. To see the latter, recall that no descendant of $R_i$ in $G_t$ (among which is $d$) is in $S$. Thus, $\rho^{Ad}_{\alpha} \cup \rho^{dE}_{\alpha}$ is an $S$-active route between $A$ and $B$ in $G_{\alpha}$.

\item[Case 3.2.6] Assume that $E \in S$. Let $\rho^{dB}_{\alpha}$ and $\rho^{Ed}_{\alpha}$ denote the routes resulting from reversing $\rho^{Bd}_{\alpha}$ and $\rho^{dE}_{\alpha}$. Consider the route $\rho^{Ad}_{\alpha} \cup \rho^{dB}_{\alpha}$ between $A$ and $B$ in $G_{\alpha}$. If this route is $S$-active, then we are done. If it is not $S$-active in $G_{\alpha}$, then $d$ occurs in a collider section of $\rho^{Ad}_{\alpha} \cup \rho^{dB}_{\alpha}$ that has no node in $S$. Then, we can replace each such occurence of $d$ with $\rho^{dE}_{\alpha} \cup \rho^{Ed}_{\alpha}$ and, thus construct an $S$-active route between $A$ and $B$ in $G_{\alpha}$.

\end{description}

\end{description}

\item[Case 4] Method B3 constructs $G_{t+1}$ from $G_t$ by adding an edge $X \rightarrow Y$ due to line 4 of Fbmerge. Note that $X \rightarrow Y$ occurs in $\rho^{AB}_{t+1}$. We consider two cases.

\begin{description}

\item[Case 4.1] Assume that $X \rightarrow Y$ occurs as a collider edge in $\rho^{AB}_{t+1}$. Let $\rho^{AB}_{t+1} = \rho^{AX}_{t+1} \cup X \rightarrow Y \cup \rho^{YB}_{t+1}$.\footnote{Note that maybe $A=X$ and/or $Y=B$.} Note that $X \notin S$ because, otherwise, $\rho^{AB}_{t+1}$ would not be $S$-active in $G_{t+1}$. For the same reason, $\rho^{AX}_{t+1}$ is $S$-active in $G_{t+1}$. Then, there is an $S$-active route $\rho^{AX}_{\alpha}$ between $A$ and $X$ in $G_{\alpha}$ by the induction hypothesis. 

Let $R_i$ denote the component of $G_t$ in $R$ that Fbmerge is processing when the edge $X \rightarrow Y$ gets added. Recall that Method B3 is currently considering the block $C$ of $\alpha$, and that it has previously considered all the blocks of $\alpha$ to the right of $C$ in $\alpha$. Then, $R_i$ only contains nodes from $C$ or from blocks to the left of $C$ in $\alpha$. In other words, $R_i \subseteq \cup_{j=1}^{k^*} C_j \setminus \{X, Y\}$ where $k^*$ is the smallest $k$ st $X, Y \in \cup_{j=1}^{k} C_j$ (recall that $Y \in L \subseteq C$). Therefore, $X \nci_{G_t} Y | \cup_{j=1}^{k^*} C_j \setminus \{X, Y\}$ because $X$ and $Y$ must be in $Pa_{G_t}(R_i)$ for line 4 of Fbmerge to add an edge $X \rightarrow Y$. Then, $X$ and $Y$ are adjacent in $G_{\alpha}$ because, otherwise, $X \ci_{G_{\alpha}} Y | \cup_{j=1}^{k^*} C_j \setminus \{X, Y\}$ which would contradict that $I(G_{\alpha}) \subseteq I(G_t)$. In fact, $X \rightarrow Y$ is in $G_{\alpha}$. To see it, recall that Method B3 is currently considering the block $C$ of $\alpha$, and that it has previously considered all the blocks of $\alpha$ to the right of $C$ in $\alpha$. Then, no block to the left of $L$ in $\beta$ has a node from $C$ or from a block to the right of $C$ in $\alpha$. Note that $X$ is to the left of $R$ in $\beta$, because $\beta$ is consistent with $G_t$. Then, $X$ is to the left of $L$ in $\beta$, because $L$ is the left neighbor of $R$ in $\beta$ and $X \notin L$. Thus, $X \rightarrow Y$ is in $G_{\alpha}$ because $Y \in L \subseteq C$. We now consider two cases.

\begin{description}

\item[Case 4.1.1] Assume that $\rho^{YB}_{t+1} = Y - \ldots - Y \leftarrow X \cup \rho^{XB}_{t+1}$. Note that a node in $Y - \ldots - Y$ is in $S$ because, otherwise, $\rho^{AB}_{t+1}$ would not be $S$-active in $G_{t+1}$. For the same reason, $\rho^{XB}_{t+1}$ is $S$-active in $G_{t+1}$. Then, there is an $S$-active route $\rho^{XB}_{\alpha}$ between $X$ and $B$ in $G_{\alpha}$ by the induction hypothesis. Note that $Y - \ldots - Y$ is in $G_{\alpha}$. To see it, note that the adjacencies in $Y - \ldots - Y$ are preserved in $G_{\alpha}$ since $I(G_{\alpha}) \subseteq I(G_t)$. Moreover, these adjacencies correspond to undirected edges in $G_{\alpha}$, because the nodes in $Y - \ldots - Y$ are in $L$ since $Y \in L$ and, thus, they are in the same block of $G_{\alpha}$ since $L \subseteq C$. Then, $\rho^{AX}_{\alpha} \cup X \rightarrow Y - \ldots - Y \leftarrow X \cup \rho^{XB}_{\alpha}$ is an $S$-active route between $A$ and $B$ in $G_{\alpha}$.

\item[Case 4.1.2] Assume that $\rho^{YB}_{t+1} = Y - \ldots - Z \leftarrow W \cup \rho^{WB}_{t+1}$.\footnote{Note that maybe $Y=Z$, $W=X$ and/or $W=B$. Note that $Y \neq Z$ or $W \neq X$, because the case where $Y=Z$ and $W=X$ is covered by Case 4.1.1.} Note that $W \notin S$ and a node in $Y - \ldots - Z$ is in $S$ because, otherwise, $\rho^{AB}_{t+1}$ would not be $S$-active in $G_{t+1}$. For the same reason, $\rho^{WB}_{t+1}$ is $S$-active in $G_{t+1}$. Then, there is an $S$-active route $\rho^{WB}_{\alpha}$ between $W$ and $B$ in $G_{\alpha}$ by the induction hypothesis. Note that $Y - \ldots - Z$ is in $G_{\alpha}$. To see it, note that the adjacencies in $Y - \ldots - Z$ are preserved in $G_{\alpha}$ since $I(G_{\alpha}) \subseteq I(G_t)$. Moreover, these adjacencies correspond to undirected edges in $G_{\alpha}$, because the nodes in $Y - \ldots - Z$ are in $L$ since $Y \in L$ and, thus, they are in the same block of $G_{\alpha}$ since $L \subseteq C$. Moreover, note that $Z$ and $W$ are adjacent in $G_{\alpha}$ since $I(G_{\alpha}) \subseteq I(G_t)$. In fact, $Z \leftarrow W$ is in $G_{\alpha}$. To see it, recall that no block to the left of $L$ in $\beta$ has a node from $C$ or from a block to the right of $C$ in $\alpha$. Note that $W$ is to the left of $L$ in $\beta$, because $\beta$ is consistent with $G_t$. Thus, $Z \leftarrow W$ is in $G_{\alpha}$ since $Z \in L \subseteq C$. Then, $\rho^{AX}_{\alpha} \cup X \rightarrow Y - \ldots - Z \leftarrow W \cup \rho^{WB}_{\alpha}$ is an $S$-active route between $A$ and $B$ in $G_{\alpha}$.

\end{description}

\item[Case 4.2] Assume that $X \rightarrow Y$ occurs as a non-collider edge in $\rho^{AB}_{t+1}$. The proof of this case is the same as that of Case 3.2, with the only exception that $X - Y$ should be replaced by $X \rightarrow Y$.

\end{description}

\end{description}

\end{proof}

\begin{figure}[t]
\centering
\small
\begin{tabular}{rl}
\hline
\\
& \underline{Method G2H($G$, $H$)}\\
\\
& /* Given two CGs $G$ and $H$ st $I(H) \subseteq I(G)$, the algorithm transforms $G$ into $H$\\ 
& by a sequence of directed and undirected edge additions and feasible splits and\\
& mergings st after each operation in the sequence $G$ is a CG and $I(H) \subseteq I(G)$ */\\
\\
1 & Let $\alpha$ denote a chain that is consistent with $H$\\
2 & Method B3($G$, $\alpha$)\\
3 & Add to $G$ the edges that are in $H$ but not in $G$\\
\\
\hline
\\
\end{tabular}
\caption{Method G2H.}\label{fig:methodg2h}
\end{figure}

We are now ready to prove the main result of this paper, namely that the extension of Meek's conjecture to CGs is true. The proof is constructive in the sense that we give an algorithm that constructs a valid sequence of operations. The pseudocode of our algorithm, called Method G2H, can be seen in Figure \ref{fig:methodg2h}. The following theorem proves that Method G2H is correct.

\begin{thm}
Given two CGs $G$ and $H$ st $I(H) \subseteq I(G)$, Method G2H($G$, $H$) transforms $G$ into $H$ by a sequence of directed and undirected edge additions and feasible splits and mergings st after each operation in the sequence $G$ is a CG and $I(H) \subseteq I(G)$.
\end{thm}

\begin{proof}

Note from line 1 that $\alpha$ denotes a chain that is consistent with $H$. Let $G_{\alpha}$ denote the MI map of $I(G)$ relative to $\alpha$. Recall that $G_{\alpha}$ is guaranteed to be unique by Lemma \ref{lem:unique}, because $I(G)$ is a graphoid. Note that $I(H) \subseteq I(G)$ implies that $G_{\alpha}$ is a subgraph of $H$. To see it, note that $I(H) \subseteq I(G)$ implies that we can obtain a MI map of $I(G)$ relative to $\alpha$ by just removing edges from $H$. However, $G_{\alpha}$ is the only MI map of $I(G)$ relative to $\alpha$.

Then, it follows from the proof of Lemma \ref{lem:correctness} that line 2 transforms $G$ into $G_{\alpha}$ by a sequence of directed and undirected edge additions and feasible splits and mergings, and that after each operation in the sequence $G$ is a CG and $I(G_{\alpha}) \subseteq I(G)$. Thus, after each operation in the sequence $I(H) \subseteq I(G)$ because $I(H) \subseteq I(G_{\alpha})$ since, as shown, $G_{\alpha}$ is a subgraph of $H$. Finally, line 3 transforms $G$ from $G_{\alpha}$ to $H$ by a sequence of edge additions. Of course, after each edge addition $G$ is a CG and $I(H) \subseteq I(G)$ because $G_{\alpha}$ is a subgraph of $H$.

\end{proof}

\section*{Acknowledgments}

We thank Dr. Jens D. Nielsen and Dag Sonntag for proof-reading this manuscript. This work is funded by the Center for Industrial Information Technology (CENIIT) and a so-called career contract at Link\"oping University, and by the Swedish Research Council (ref. 2010-4808).

\end{document}